\documentclass{jmlr-sanitized}

\def\ind{\mathbf{1}}
\def\N{\mathbb{N}_{+}}
\def\R{\mathbb{R}}
\def\E{\mathbf{E}}
\def\Prob{\mathbf{P}}
\def\c{\mathsf{c}}
\def\hypothesis{($\mathrm{H}$)}
\def\induction{(\textasteriskcentered)\xspace}
\def\half{\tfrac{1}{2}}

\title{\vspace{-18pt}On the Impossibility of Learning the Missing Mass\\\vspace{6pt}}
\usepackage{times}

 \author{\Name{Elchanan Mossel} \thanks{Supported by NSF grants DMS 1106999 and CCF 1320105, by DOD ONR grant N00014-14-1-0823, and by grant 328025 from the Simons Foundation.}\footnotemark[3]
 \Email{mossel@gmail.com}\\
 \addr Department of Statistics, The Wharton School, University of Pennsylvania \\
 Departments of Statistics and Computer Science, University of California, Berkeley
 \AND  
 \Name{Mesrob I. Ohannessian}\thanks{Supported by funds from the California Institute for Telecommunications and Information Technology (Calit2).} \thanks{This work was conducted when both authors were visiting the Information Theory Program, Jan 13 -- May 15, 2015, at the Simons Institute for the Theory of Computing, University of California, Berkeley.}
 \Email{mesrob@gmail.com}\\
 \addr University of California, San Diego
 }

\begin{document}

\maketitle

\vspace{-12pt}
\begin{abstract}
This paper shows that one cannot learn the probability of rare events without imposing further structural assumptions. The event of interest is that of obtaining an outcome outside the coverage of an i.i.d.\ sample from a discrete distribution. The probability of this event is referred to as the ``missing mass''. The impossibility result can then be stated as: the missing mass is not distribution-free PAC-learnable in relative error. The proof is semi-constructive and relies on a coupling argument using a dithered geometric distribution. This result formalizes the folklore that in order to predict rare events, one necessarily needs distributions with ``heavy tails''.
\end{abstract}

\vspace{6pt}


\section{Introduction}

Given data consisting of $n$ i.i.d.\ samples $X_1,\cdots,X_n$ from an unknown distribution $p$ over the integers $\N$, we traditionally compute the \emph{empirical distribution}:
$$
  \hat{p}_n(x) := \frac{1}{n} \sum_{i=1}^{n} \ind\{X_i=x\}.
$$

To estimate the probability $p(E)$ of an event $E\subset \N$, we could use $\hat{p}_n(E)$. This works well for abundantly represented events, but not as well for rare events. An unequivocally rare event is the set of symbols that are \emph{missing} in the data,
$$
  E_n := \{x\in \N: \hat{p}(x)=0\}.
$$
The probability of this (random) event is denoted by the \emph{missing mass}:
$$
  M_n(X_1,\cdots,X_n) := p(E_n) = \sum_{x\in\N} p(x) \ind\{\hat{p}(x)=0\}.
$$

The question we strive to answer in this paper is: ``Can we learn the missing mass when $p$ is an arbitrary distribution on $\N$?'' Definition \ref{def:PAC-learning} phrases this precisely in the PAC-learning framework.

\begin{definition} \label{def:PAC-learning}
An \emph{estimator} is a sequence of functions $\hat{M}_n(x_1,\cdots,x_n): \N^n \to [0,1]$. We say that an estimator \emph{PAC-learns} the missing mass in relative error with respect to a family $\mathcal P$ of distributions, if for every $p\in\mathcal P$ and every $\epsilon,\delta>0$ there exists $n_0(p,\epsilon,\delta)$ such that for all $n>n_0(p,\epsilon,\delta)$:
$$
  \Prob_p\left\{\left|\frac{\hat{M}_n(X_1,\cdots,X_n)}{M_n(X_1,\cdots,X_n)}-1\right|<\epsilon\right\}>1-\delta.
$$
The learning is said to be \emph{distribution-free}, if $\mathcal P$ consists of \emph{all} distributions on $\N$.
\end{definition}

Our question thus becomes: Can we distribution-free PAC-learn the missing mass in relative error? It is obvious that the empirical estimator $\hat p(E_n)$ gives us the trivial answer of $0$, and cannot learn the missing mass. A popular alternative is the Good-Turing estimator of the missing mass, which is the fraction of singletons in the data:
$$
  G_n := \sum_{x\in\N} \frac{1}{n} \ind\{n\hat{p}(x)=1\}.
$$

The Good-Turing estimator has many interpretations. Its original derivation by \citet{Good1953} uses an empirical-Bayes perspective. It can also be thought of as a leave-one-out cross-validation estimator, which contributes to the missing set if and only if the holdout appears exactly once in the data. Fundamentally, $G_n$ derives its form and its various properties from the simple fact that:
$$
  \E[G_n]= \sum_{x\in\N} p(x)(1-p(x))^{n-1} = \E[M_{n-1}].
$$

A study of $G_n$ in the PAC-learning framework was first undertaken by \citet{McAllesterSchapire2000} and continued later by \citet{McAllesterOrtiz2003}. Some further refinement and insight was also given later by \citet{BerendKontorovich}. These works focused on additive error. \citet{OhannessianDahleh2012} shifted the attention to relative error, establishing the PAC-learning property of the Good-Turing estimator with respect to the family of heavy-tailed (roughly power-law) distributions, e.g. $p(x) \propto x^{-1/\alpha}$ with $\alpha\in(0,1)$. This work also showed that Good-Turing \emph{fails} to learn the missing mass for geometric distributions, and therefore does not achieve distribution-free learning. More recently, \citet{BenHamou2014} provide a comprehensive and tight set of concentration inequalities, which can be interpreted in the current PAC framework, and which further demonstrate that Good-Turing can PAC-learn with respect to heavier-than-geometric light tails, e.g. the family that includes $p(x) \propto 2^{-x^\alpha}$ with $\alpha\in(0,1)$ in addition to power-laws.

These results leave open the important question of whether there exists some \emph{other} estimator that can PAC-learn the missing mass in relative error in a distribution-free fashion (i.e. for \emph{any} distribution $p$). Our main contribution is to prove that there are no such estimators.

The first insight to glean from this impossibility result is that one is justified to use further structural assumptions when learning about rare events. Furthermore, the proof relies on an implicit construction that uses a dithered geometric distribution. In doing so, it shows that the failure of the Good-Turing estimator for light-tailed distributions is not a weakness of the procedure, but is rather due to a fundamental barrier. Conversely, the success of Good-Turing for heavier-than-geometric and power laws shows its universality, in some restricted sense. In particular, in concrete support to folklore \citep[e.g.][]{Taleb2008}, we can state that for estimating probabilities of rare events, heavy tails are both necessary and sufficient.

The paper is organized as follows. In Section \ref{sec:main}, we present our main result, with a detailed exposition of the proof. In Section \ref{sec:discussion} we give an immediate extension to continuous tail estimation, show that parametric light-tailed learning is possible, comment further on the Good-Turing estimator, and concisely place this result in the context of a chief motivating application, that of computational linguistics. Lastly, we conclude in Section \ref{sec:summary} with a summary and open questions.

\subsubsection*{Notation}
We use the shorthand $M_n=M_n(X_1,\cdots,X_n)$ for the missing mass and $\hat{M}_n=\hat{M}_n(X_1,\cdots,X_n)$ for its estimator, keeping implicit their dependence on the samples and, in the case of $M_n$, on the distribution $p$.

\section{Main Result} \label{sec:main}

Our main result is stated as follows. The rest of this section is dedicated to its detailed proof.

\begin{theorem} \label{thm:main}
There exists a positive $\epsilon>0$ and a strictly increasing sequence $(n_k)_{k=1,2,\cdots}$, such that for every estimator $\hat{M}_n$ there exists a distribution $p^\star$, such that for all $k$:
\begin{equation} \label{eq:divergence}
  \Prob_{p^\star}\left\{\left|\frac{\hat{M}_{n_k}}{M_{n_k}}-1\right|>\epsilon\right\}>\epsilon.
\end{equation}
In particular, it follows that it is impossible to perform distribution-free PAC-learning of the missing mass in relative error.
\end{theorem}

\begin{remark}
Our proof below implies the statement of the theorem with $\epsilon=10^{-4}$ and $n_k=6.5 \cdot 2^k$, but we did not make an honest effort to optimize these parameters.
\end{remark}

\subsection{Proof Outline}

Consider the family $\mathcal{P}_{\beta,m}$ of $\beta$-dithered geometric$(\half)$ distributions, where the mass of each outcome beyond a value $m$ of a $\textrm{geometric}(\half)$ random variable is divided between two sub-values, with a fraction $\beta$ in one and $1-\beta$ in the other. More precisely:
\begin{definition}
The $\beta$-dithered geometric$(\half)$ family is a collection of distributions parametrized by the dithering choices $\theta\in\{\beta,1-\beta\}^\N$, with $\beta\in(0,\half)$, as follows:
\begin{multline}\label{eq:dithered-family}
  \mathcal{P}_{\beta,m} = \left\{ p_\theta : p_\theta(x) = \frac{1}{2^x}, x=1,\cdots,m; \right. \\ \left. p_\theta(m+2j-1) = \frac{\theta_j}{2^{m+j}},\ p_\theta(m+2j) = \frac{1-\theta_j}{2^{m+j}},\ j\in\N,\ \theta\in\{\beta,1-\beta\}^{\N}  \right\}.
\end{multline}
\end{definition}

The intuition of the proof of Theorem \ref{thm:main} is that within such light-tailed families, two distributions may have very similar samples and thus estimated values, yet have significantly different true values of the missing mass. This follows the general methodology of many statistical lower bounds. We now state the outline of the proof. We choose a subsequence of the form $n_k=C2^k$. We set $\beta=1/4$, $m=1$, and $C=6.5$. The value of $\epsilon>0$ is made explicit in the proof, and depends only on these choices. We proceed by induction.
\begin{itemize}
  \item We show that there exists $\theta^\star_1$ such that for all $\theta$ with $\theta_1=\theta^\star_1$ we have for $n=n_1$:
    \begin{equation} \label{eq:bounded-away}
    \Prob_{p_\theta}\left\{\left|\tfrac{\hat{M}_{n}}{M_{n}}-1\right|>\epsilon\right\}>\epsilon.
    \end{equation}
  \item Then, at every step $k>1$ :
  \begin{itemize}
    \item[\hypothesis] We start with $(\theta^\star_1,\cdots,\theta^\star_{k-1})$ such that for all $\theta$ with $(\theta_1,\cdots,\theta_{k-1}) =$ $(\theta^\star_1,\cdots,\theta^\star_{k-1})$, Inequality \eqref{eq:bounded-away} holds for $n=n_1,\cdots,n_{k-1}$.
   \item[\induction] We then show that it must be that for at least one of $\tilde \theta =\beta$ or $\tilde \theta=1-\beta$, for all $\theta$ with $(\theta_1,\cdots,\theta_k) = (\theta^\star_1,\cdots,\theta^\star_{k-1},\tilde \theta)$, Inequality \eqref{eq:bounded-away} holds additionally for $n=n_k$. We select $\theta^\star_k$ to be the corresponding $\tilde\theta$.
  \end{itemize}
  \item This induction produces an infinite sequence $\theta^\star\in\{\beta,1-\beta\}^\N$, and the desired distribution in Theorem \ref{thm:main} can be chosen as $p^\star=p_{\theta^\star}$, since it is readily seen to satisfy the claim for each $n_k$, by construction.
\end{itemize}

\subsection{Proof Details}

We skip the proof of the base case, since it is mostly identical to that of the induction step. Therefore, in what follows we are given $(\theta^\star_1,\cdots,\theta^\star_{k-1})$ by hypothesis \hypothesis, and we would like to prove that the selection in \induction can always be done. Let us denote the two choices of parameters by
$$
  \theta:= (\theta^\star_1,\theta^\star_{k-1},\beta ,\theta_{k+1},\cdots),
$$
and
$$
  \theta':=(\theta^\star_1,\theta^\star_{k-1},1-\beta,\theta'_{k+1},\cdots),
$$
and let us refer to  $(\theta_{k+1},\cdots)$ and $(\theta'_{k+1},\cdots)$ by the \emph{trailing parameters}. What we show in the remainder of the proof is that with two arbitrary sets of trailing parameters, we cannot have two simultaneous violations of Inequality \eqref{eq:bounded-away} (for both $\theta$ and $\theta'$). That is, we cannot have both:
\begin{equation}\label{eq:failures}
\Prob_{p_\theta}\left\{\left|\tfrac{\hat{M}_{n_k}}{M_{n_k}}-1\right|>\epsilon\right\}\boldsymbol{<}\epsilon
\quad\textrm{ and }\quad
\Prob_{p_{\theta'}}\left\{\left|\tfrac{\hat{M}_{n_k}}{M_{n_k}}-1\right|>\epsilon\right\}\boldsymbol{<}\epsilon.
\end{equation}

This is shown in Lemma \ref{lemma:induction}, in the last portion of this section. To see why this is sufficient to show that the selection in \induction can be done, consider first the case that Inequality \eqref{eq:bounded-away} is upheld for both $\theta$ and $\theta'$ with any two sets of trailing parameters. In this case we can arbitrarily choose $\theta^\star_k$ to be either $\beta$ or $1-\beta$, since the induction step is satisfied. We can therefore focus on the case in which this fails. That is, for either $\theta$ or $\theta'$ a choice of trailing parameters can be made such that Inequality \eqref{eq:bounded-away} with $n=n_k$ is \emph{not} satisfied, and therefore one of the two cases in \eqref{eq:failures} holds [say, for example, for $\theta$]. Fix the corresponding trailing parameters [in this example, $(\theta_{k+1},\cdots)$]. Then, for \emph{any} choice of the \emph{other} set of trailing parameters [in this example, $(\theta'_{k+1},\cdots)$], Lemma \ref{lemma:induction} precludes a violation of Inequality \eqref{eq:bounded-away} for $n=n_k$ by the other choice [in this example, $\theta'$]. Therefore this choice can be selected for $\theta_k$ [in this example, $\theta_k=1-\beta$.]

By using the \emph{coupling} device and restricting ourselves to a \emph{pivotal event}, we formalize the aforementioned intuition that the estimator may not distinguish between two separated missing mass values, and deduce that both statements in \eqref{eq:failures} cannot hold simultaneously.

\subsubsection*{Coupling}

\begin{definition} \label{def:coupling}
A \emph{coupling} between two distributions $p$ and $p'$ on $\N$ is a joint distribution $q$ on $\N^2$, such that the first and second marginal distributions of $q$ revert back to $p$ and $p'$ respectively.
\end{definition}

Couplings are useful because probabilities of events on each side may be evaluated on the joint probability space, while forcing events of interest to occur in an orchestrated fashion. Going back to our induction step and the specific choices $\theta$ and $\theta'$ with arbitrary trailing parameters, we perform the following coupling:
\begin{equation} \label{eq:coupling}
q(x,x')=\left\{
\begin{array}{lcl}
p_\theta(x)=p_{\theta'}(x') &;& \textrm{if } x=x' < m+2k-1; \\
\beta/2^{m+k} &;& \textrm{if } x=x'=m+2k-1,\textrm{or if } x=x'=m+2k, \\
(1-2\beta)/2^{m+k} &;& \textrm{if } x=m+2k,\ x'=m+2k-1; \\
p_\theta(x)p_{\theta'}(x')/2^{m+k} &;& \textrm{if } x,x' > m+2k; \\
0 &;& \textrm{otherwise.}
\end{array}
\right.
\end{equation}

It is easy to verify that $q$ in Equation \eqref{eq:coupling} is a coupling between $p_\theta$ and $p_{\theta'}$ as in Definition \ref{def:coupling}. Note the resulting outcomes. If $X,X'$ are generated according to $q$, then if either is in $\{1, \cdots, m+2k-2\}$ then both values are \emph{identical}. If either is in $\{m+2k+1,\cdots\}$ then so is the other, but otherwise the two values are conditionally independent. If either is in $\{m+2k-1,m+2k\}$, so is the other, and the conditional probability is given by:

\begin{center}
\begin{tabular}{|r||c|c|}
 \hline
$x, x'$ & $m+2k-1$ & $m+2k$ \\
 \hline \hline
 $m+2k-1$ & $\beta$ & $0$\\
 \hline
 $m+2k$ & $1-2\beta$ & $\beta$ \\
 \hline
\end{tabular}
\end{center}

Now consider coupled data $(X_i,X'_i)_{i=1,\cdots,n}$ generated as i.i.d.\ samples from $q$. It follows that, marginally, the $X$-sequence is i.i.d.\ from $p_\theta$, and so is the $X'$-sequence from $p_{\theta'}$. Any event $B$ that is exclusively $X$-measurable or $B'$ that is exclusively $X'$-measurable has the same probability under the coupled measure. That is,
$$
\Prob_{p_\theta}(B)=\Prob_q(B):=q^n(B\times \N^n)
$$
and
$$
\Prob_{p_{\theta'}}(B')=\Prob_q(B'):=q^n(\N^n \times B').
$$
In what follows we work only with coupled data, and use simply the shorthand $\Prob$ to mean $\Prob_q$.

\subsubsection*{Pivotal Event}

The event we would like to work under is that of the coupled samples being identical, while exactly covering the range $1,\cdots,m+2k-1$:
\begin{equation} \label{eq:pivotal}
A_k = \bigcap_{i=1}^{n_k} \{ X_i = X'_i\} ~~\cap~~ \Big\{\{X_1,\ldots,X_{n_k}\} = \{1,\cdots,m+2k-1\}\Big\}.
\end{equation}

The reason $A_k$ interests us is that it encapsulates the aforementioned intuition.
\begin{lemma} \label{lemma:A}
Under event $A_k$, the coupled missing masses are distinctly separated,
$$
\frac{M_{n_k}}{M'_{n_k}} = \frac{2-\beta}{1+\beta},
$$
while any estimator cannot distinguish the coupled samples,
$$
\hat M_{n_k} = \hat M'_{n_k}.
$$
\end{lemma}

\begin{proof}
The confusion of any estimator is simply due to the fact that under $A_k$, the coupling forces all samples to be identical $X_i=X'_i$, for all $i=1,\cdots,n_k$.
Thus $\hat M_{n_k} = \hat M'_{n_k}$, since estimators only depend on the samples and not the probabilities.

The missing masses, on the other hand, do depend on both the samples and the probabilities and thus they differ. But the event $A_k$ makes the set of missing symbols simply the tail $m+2k,m+2k+1,\cdots$, so we can compute the missing masses exactly:
$$
  M_{n_k} = p_\theta(m+2k)+\sum\nolimits_{x=m+2k+1}^\infty p_\theta(x)=\frac{1-\theta_k}{2^{m+k}}+\frac{1}{2^{m+k}}=(2-\beta)2^{-m-k},\textrm{ and}
$$
$$
  M'_{n_k} = p_{\theta'}(m+2k)+\sum\nolimits_{x=2k+1}^\infty p_{\theta'}(m+x)=\frac{1-\theta'_k}{2^{m+k}}+\frac{1}{2^{m+k}}=(1+\beta)2^{-m-k},
$$
and the claim follows.
\end{proof}

We now show that $A_k$ has always a positive probability, bounded away from zero.
\begin{lemma} \label{lemma:P(A)}
For $\beta=1/4$, $m=1$, $C=6.5$ and $n_k=C 2^k$, there exists a positive absolute constant $\eta>0$ such that for all $k$, $\Prob(A_k)>\eta$. We can explicitly set $\eta=2\cdot 10^{-4}$.
\end{lemma}

\begin{proof}
Note that $A_k$ in Equation \eqref{eq:pivotal} overspecifies the event. In fact, only forcing the exact coverage of $1,\cdots,m+2k-1$ is sufficient, since this implies in turn that the coupled samples are identical. This is evident for values in $1,\cdots,m+2k-2$. But since $m+2k$ is not allowed in this event, it also holds for the value $m+2k-1$. We can then write $A_k=A_{k,1}\cap A_{k,2}$, dividing the exact coverage to the localization in the range and the representation of each value by at least one sample:
$$
\begin{array}{lcl}
A_{k,1} = \left\{ \bigcup\nolimits_{i=1}^{n_k} \{X_i\} \subseteq \{1,\cdots,m+2k-1\} \right\}&&\ \textrm{(localization),}\\
A_{k,2} = \left\{ \bigcup\nolimits_{i=1}^{n_k} \{X_i\} \supseteq \{1,\cdots,m+2k-1\} \right\}&&\ \textrm{(representation).}
\end{array}
$$

Let $\alpha$ be the probability of $(x,x')$ being in $\{(1,1),\cdots,(m+2k-1,m+2k-1)\}$. From the coupling in Equation \eqref{eq:coupling} and the structure of the dithered family in Equation \eqref{eq:dithered-family}, we see that for up to $m+2k-2$ this probability sums up to the $m+k-1$ first terms of a geometric$(\half)$, and for $(m+2k-1,m+2k-1)$ the coupling assigns it $\beta/2^{m+k}$, thus:
$$
  \alpha = \sum\nolimits_{x=1}^{2k-1} q(x,x) = 1-\frac{1}{2^{m+k-1}} + \frac{\beta}{2^{m+k}}.
$$

We can then explicitly compute:
$$
\Prob(A_{k,1}) = \alpha^{n_k} = \left(1-\frac{1}{2^{m+k-1}}+\frac{\beta}{2^{m+k}}\right)^{n_k} =: \eta_1(k).
$$

Meanwhile, note that conditionally on $A_{k,1}$, the occurrence probabilities on $\{(1,1),\cdots,(m+2k-1,m+2k-1)\}$ are simply normalized by $\alpha$. By using a union bound on the complement of $A_{k,2}$ (the event of at least one of these values not appearing), we then have that:

\begin{eqnarray*}
\Prob(A_{k,2}|A_{k,1})
&\geq& 1 - \sum\nolimits_{x=1}^{m+2k-1} \left[1-q(x,x)/\alpha\right]^{n_k} \\
&\geq& 1 - \sum\nolimits_{x=1}^{m+2k-1} \left[1-q(x,x)\right]^{n_k} \\
&=&    1-\sum\nolimits_{x=1}^{m}\left(1-\tfrac{1}{2^ x}\right)^{n_k}\\
&& \qquad \qquad -\sum\nolimits_{j=1}^{k-1} \left[\left(1-\tfrac{\beta}{2^{m+j}}\right)^{n_k}+\left(1-\tfrac{1-\beta}{2^{m+j}}\right)^{n_k}\right] - \left(1-\tfrac{\beta}{2^{m+k}}\right)^{n_k} \\
&\geq& 1-\sum\nolimits_{x=1}^{m}\left(1-\tfrac{1}{2^ x}\right)^{n_k}-2\sum\nolimits_{j=1}^{k-1} \left(1-\tfrac{\beta}{2^{m+j}}\right)^{n_k} - \left(1-\tfrac{\beta}{2^{m+k}}\right)^{n_k} =: \eta_2(k).
\end{eqnarray*}

Therefore,
$$
\Prob(A_k) = \Prob(A_{k,1}\cap A_{k,2}) = \Prob(A_{k,1})\Prob(A_{k,2}|A_{k,1}) \geq \eta_1(k) \eta_2(k) \geq \inf_{k\geq 1} \eta_1(k) \eta_2(k) =: \eta.
$$

We now use our choices of $\beta=1/4$, $m=1$, $C=6.5$, and $n_k=C2^k$, to bound this worst-case $\eta$. In particular, we can verify that $\eta \geq 2\cdot 10^{-4}$, and it follows as claimed that the pivotal event has always a probability bounded away from zero.
\end{proof}

\subsubsection*{Induction Step}

We now combine all the elements presented thus far to complete the proof of Theorem \ref{thm:main} by establishing the following claim, which we have shown in the beginning of the detailed proof section to be sufficient for the validity of the induction step. In particular, we restate Equation \eqref{eq:failures} under the coupling of Equation \eqref{eq:coupling}.

\begin{lemma} \label{lemma:induction}
Let
$$
  \theta:= (\theta^\star_1,\theta^\star_{k-1},\beta ,\theta_{k+1},\cdots),
\ \textrm{ and }\ 
  \theta':=(\theta^\star_1,\theta^\star_{k-1},1-\beta,\theta'_{k+1},\cdots),
$$
with \emph{arbitrary} trailing parameters $(\theta_{k+1},\cdots)$ and $(\theta'_{k+1},\cdots)$. Let $q$ be the coupling of Equation \eqref{eq:coupling}, and let $B_k=\left\{\left|\hat{M}_{n_k}/M_{n_k}-1\right|>\epsilon\right\}$ and $B_k'=\left\{\left|\hat{M}'_{n_k}/M'_{n_k}-1\right|>\epsilon\right\}$. Then given our choices of $\beta=1/4$, $m=1$, $C=6.5$ and $n_k=C 2^k$, if $\epsilon<10^{-4}$ we cannot simultaneously have
$$
\Prob_q(B_k)<\epsilon\ \textrm{ and }\ \Prob_q(B'_k)<\epsilon.
$$
\end{lemma}
\begin{proof}
Note that this choice of $\epsilon$ means that $\epsilon<\eta/2$, where $\eta$ is as in Lemma \ref{lemma:P(A)}. Recall the pivotal event $A_k$, and assume, for the sake of contradiction, that both probability bounds $\Prob(B_k)<\epsilon$ and $\Prob(B'_k)<\epsilon$ hold. Note that if $B_k^\c$ holds, it means that
\begin{equation} \label{eq:B}
\hat{M}_{n_k}/M_{n_k} \in (1-\epsilon,1+\epsilon),
\end{equation}
and similarly if $B^{\prime \c}_k$ holds, it means that
\begin{equation} \label{eq:B'}
\hat{M}'_{n_k}/M'_{n_k} \in (1-\epsilon,1+\epsilon).
\end{equation}

By making our hypothesis, we are asserting that these events have high probabilities, $1-\epsilon$, under both $p_\theta$ and $p_{\theta'}$ distributions, and that thus the estimator is effectively $(1\pm\epsilon)$-close to the true value of the missing mass. Yet, we know that this would be violated under the pivotal event, which occurs with positive probability. We now formalize this contradiction.

By Lemma \ref{lemma:P(A)}, we have that:
\begin{equation} \label{eq:absurd}
\left.
\begin{aligned}
\Prob(B_k|A_k) &= \frac{\Prob(A_k\cup B_k)}{\Prob(A_k)} \leq \frac{\Prob(B_k)}{\Prob(A_k)} \leq \frac{\epsilon}{\eta}\\
\Prob(B'_k|A_k) &= \frac{\Prob(A_k\cup B'_k)}{\Prob(A_k)} \leq \frac{\Prob(B'_k)}{\Prob(A_k)} \leq \frac{\epsilon}{\eta}
\end{aligned}
\ \ 
\right\}
\ 
\Rightarrow
\ \ 
\Prob(B_k^\c\cap B^{\prime \c}_k|A_k) \geq 1-2 \frac{\epsilon}{\eta} > 0,
\end{equation}
where the last inequality is strict, by the choice of $\epsilon<\eta/2$.

On the other hand, recall that by Lemma \ref{lemma:A} under $A_k$ we have:
$$
\hat{M}_{n_k} = \hat{M}'_{n_k} \quad\textrm{and}\quad \frac{M_{n_k}}{M'_{n_k}} = \frac{2-\beta}{1+\beta} = \tfrac{7}{5}.
$$

By combining this with Equations \eqref{eq:B} and \eqref{eq:B'}, we can now see that if $\frac{1+\epsilon}{1-\epsilon}<\frac{7}{5}$, which is satisfied by any choice of $\epsilon<1/6$, in particular ours, then if $B_k^\c$ occurs, then $B'_k$ occurs, and conversely if $B^{\prime \c}_k$ occurs then $B_k$ occurs. For example, say $B_k^\c$ occurs, then $\hat{M}_{n_k}/M_{n_k}<(1+\epsilon)$:
$$
\frac{\hat{M}'_{n_k}}{M'_{n_k}} = \frac{\hat{M}_{n_k}}{\tfrac{7}{5}M_{n_k}} = \tfrac{5}{7}(1+\epsilon) < 1-\epsilon,
$$
implying that Equation \eqref{eq:B'} is not satisfied, thus $B'_k$ occurs. The end result is that under event $A_k$, $B_k^\c$ and $B^{\prime \c}_k$ cannot occur at the same time, and thus:
$$
\Prob(B_k^\c\cap B^{\prime \c}_k|A_k)=0.
$$
This contradicts the bound in \eqref{eq:absurd}, and establishes the lemma.
\end{proof}

\section{Discussions} \label{sec:discussion}

\subsection{Generalization to continuous tails}

A closely related problem to learning the missing mass is that of estimating the tail of a probability distribution. In the simplest setting, the data consists of $Y_1,\cdots,Y_n$ that are i.i.d.\ samples from a continuous distribution on $\R$. Let $F$ be the cumulative distribution function. The task in question is that of estimating the tail probability
$$
  W_n = 1-F\left(\max_{i=1}^n ~Y_i \right),
$$
that is the probability that a new sample exceeds the maximum of all samples seen in the data.

One can immediately see the similarity with the missing mass problem, as both problems concern estimating probabilities of underrepresented events. We can use essentially the same PAC-learning framework given by Definition \ref{def:PAC-learning}, and prove a completely parallel impossibility result.

\begin{theorem} \label{thm:tail}
For every estimator $\hat{W}_n$ of $W_n$ there exists a distribution $F^\star$, a positive value $\epsilon>0$, and a subsequence $(n_k)_{k=1,2,\cdots}$, such that for all $k$:
\begin{equation*}
  \Prob_{F^\star}\left\{\left|\frac{\hat{W}_{n_k}}{W_{n_k}}-1\right|>\epsilon\right\}>\epsilon.
\end{equation*}
In particular, it follows that it is impossible to perform distribution-free PAC-learning of the tail probability in relative error.
\end{theorem}

\begin{proof}[Sketch]
Recall that in the proof of Theorem \ref{thm:main}, the pivotal event forced the missing mass to be a tail probability. Therefore, most of the arguments go through unchanged. Instead of dithering a geometric distribution, we dither an exponential distribution, by shifting the mass in adjacent blocks. Some of the adjustments that need to be performed concern the exact location of the samples within each block, but coarse bounds can be given by taking the extremities of each block instead.
\end{proof}

Theorem \ref{thm:tail} gives a concrete justification of why it is important to make regularity assumptions when extrapolating distribution tails. This is of course the common practice of extreme value theory, \citep[see, for example,][]{Beirlant2004}. Some impossibility results concerning the even more challenging problem of estimating the density of the maximum were already known, \citep{Beirlant1999}, but to the best of our knowledge this is the first result asserting it for tail probability estimation as well.

\subsection{Learning in various families}

\citet{BenHamou2014} (Corollary 5.3) gives a very clean characterization of a sufficient learnable family, which encompasses the one covered by \citet{OhannessianDahleh2012}.

\begin{theorem}[\cite{BenHamou2014}] \label{thm:benhamou} Let $\mathcal H$ be the family of distributions on $\N$ that satisfy
$$
\E\left[\sum_{x\in\N} \ind\{n\hat{p}_n(x)=1\}\right]=\sum_{x\in\N} n p(x)[1-p(x)]^{n-1} \to \infty.
$$
The Good-Turing estimator PAC-learns the missing mass in relative error with respect to $\mathcal H$.
\end{theorem}

Note that this theorem in the cited paper asks for an additional technical condition, but this can be relaxed. The proof relies on power moment concentration inequalities (such as Chebyshev's). For us, this is instructive because one could readily verify that the condition of Theorem \ref{thm:benhamou} fails for geometric (and dithered geometric) distributions. We can thus see that in some sense Good-Turing captures a maximal family of learnable distributions. In particular, we now know that the complement of $\mathcal H$ is not learnable.

Considering how sparse the dithered geometric family is, the failure of any estimator to learn the missing mass with respect to it may seem discouraging. (Note that Theorem \ref{thm:main} holds even if the estimator \emph{is aware} that this is the class it is paired with.) However, if we restrict ourselves to smooth parametric families within the light tails then the outlook can be brighter. We illustrate this with the case of the geometric family.

\begin{theorem}\label{thm:geometric-family} Let $\mathcal G$ be the class of geometric distributions, parametrized by $\alpha\in(0,1)$:
$$
p_\alpha(x) = (1-\alpha)\alpha^{x-1},\qquad \textrm{for } x\in\N.
$$
Let $\hat{\alpha}_n = 1-\frac{n}{\sum X_i}$ be the empirical estimator of the parameter, and define the plug-in estimator:
$$
\check{M}_n = \sum_{x\in\N} (1-\hat\alpha_n){\hat \alpha_n}^x \ind\{n\hat{p}_n(x)=0\}
$$
\vspace{-6pt}
Then $\check{M}_n$ PAC-learns the missing mass in relative error with respect to $\mathcal G$.
\end{theorem}
\begin{proof}[Sketch]
The proof consists of pushing forward the convergence of the parameter to that of the entire distribution using continuity arguments, and then specializing to the missing mass. The details can be found in the appendix.
\end{proof}

\subsection{\texorpdfstring{$N$-gram models and Bayesian perspectives}{N-gram models and Bayesian perspectives}}

One of the prominent applications of estimating the missing mass has been to computational linguistics. In that context, it is known as \emph{smoothing} and is used to estimate $N$-gram transition probabilities. The importance of accurately estimating the missing mass, and in particular in a relative-error sense, comes from the fact that $N$-grams are used to score test sentences using log-likeliehoods. Test sentences often have transitions that are never seen in the training corpus, and thus in order for the inferred log-likelihoods to accurately track the true log-likelihood, these rare transitions need to be assigned meaningful values, ideally as close to the truth as possible. As such, various forms of smoothing, including Good-Turing esimation, have become an essential ingredient of many practical algorithms, such as the popular method proposed by \citet{kneser-ney}.

In the context of $N$-gram learning, a separate Bayesian perspective was also proposed. One of the earliest to introduce this were \cite{mackay-peto} using a Dirichlet prior. This was shown to not be very effective, and we now understand that it is due to the fact that (1) the Dirichlet process produces light tails while language is often heavy-tailed and, even if it were, (2) rare probabilities are hard to learn for large light-tailed families. The natural progression of these Bayesian models led to the use of the two-parameter Poisson-Dirichlet prior \citep{pitman-yor}, which was suggested initially by \cite{teh}. It is worth remarking that a significant part of the contribution of these Bayesian models, beyond modeling the missing mass, were to introduce formal hierarchies, which is outside our scope. Concerning the missing mass, however, this line of work soon remarked that the inference techniques closely followed the Good-Turing estimator, albeit being computationally much more demanding. In light of the present work, this is not surprising since the two-parameter Poisson-Dirichlet process almost surely produces heavy-tailed distributions, and any two algorithms that learn the missing mass are bound to have the same qualitative behavior.

\section{Summary} \label{sec:summary}

In this paper, we have considered the problem of learning the missing mass, which is the probability of all unseen symbols in an i.i.d.\ draw from an unknown discrete distribution. We have phrased this in the probabilistic framework of PAC-learning. Our main contribution was to show that it is not possible to  learn the missing mass in a completely distribution-free fashion.

In other words, no single estimator can do well for all distributions. We have given a detailed account of the proof, emphasizing the intuition of how failure can occur in large light-tailed families. We have also placed this work in a greater context, through some discussions and extensions of the impossibility result to continuous tail probability estimation, and by showing that smaller, parametric, light-tailed families may be learnable.

An initial impetus for this paper and its core message is that assuming further structure can be necessary in order to learn rare events. Further structure, of course, is nothing more than a form of regularization. This is a familiar notion to the computational learning community, but for a long time the Good-Turing estimator enjoyed favorable analysis that focused on additive error, and evaded this kind of treatment. The essential ill-posedness of the problem was uncovered by studying relative error. But lower bounds cannot be deduced from the failure of particular algorithms. Our result thus completes the story, and we can now shift our attention to studying the landscape that is revealed.

The most basic set of open problems concerns establishing families that allow PAC-learning of the missing mass. We have seen in this paper some such families, including the heavy-tailed family learnable by the Good-Turing estimator, and simple smooth parametric families, learnable using plug-in estimators. How do we characterize such families more generally? The next layer of questions concerns establishing convergence rates, via both lower and upper bounds. The fact that a family of distributions allows learning does not mean that such rates can be established. This is because any estimator may be faced with arbitrarily slow convergence, by varying the distribution in the family. In other words we may be faced with a lack of uniformity. How do we control the convergence rate? Lastly, when learning is not possible, we may want to establish how gracefully an estimator can be made to fail. Understanding these limitations and accounting for them can be critical to the proper handling of data-scarce learning problems.


\bibliography{main}

\appendix
\section{Proof of Theorem \ref{thm:geometric-family}}

\paragraph{(Notation and outline)}  Let us first set some notation. Recall that the mean of the geometric distribution $p_\alpha(x)=(1-\alpha)\alpha^{x-1}$ is $\mu=\frac{1}{1-\alpha}$ and its variance is $\sigma^2=\frac{\alpha}{(1-\alpha)^2}$. Let us write the empirical mean and our parameter estimate respectively as follows:
$$
  \hat\mu_n=\frac{1}{n} \sum_{i=1}^{n} X_i, \quad \hat\alpha_n = 1-\frac{1}{\hat\mu_n}.
$$

The plug-in probability estimate can be expressed as:
$$
\check{p}_n(x) := (1-\hat\alpha_n){\hat \alpha_n}^{x-1}.
$$
Using our notation for the missing symbols, $E_n := \{x\in \N: \hat{p}(x)=0\},$ the missing mass is
$$
M_n=p_\alpha(E_n)= \sum_{x\in E_n} (1-\alpha){\alpha}^{x-1}
$$
and the suggested plug-in estimator can be written as
$$
\check{M}_n := \check{p}_n(E_n) = \sum_{x\in E_n} (1-\hat\alpha_n){\hat \alpha_n}^{x-1}.
$$

The following proof first establishes the convergence of the parameter estimate and then pushes it forward to the entire distribution,  specializing in particular to the missing mass. For the latter, we establish some basic localization properties of the punctured segment of a geometric sample coverage. This is related to the general study of gaps \citep[see, for example,][]{LouchardProdinger2008}.

We have the following elementary convergence property for the parameter.

\begin{lemma}[Parameter Convergence] \label{lemma:parameter}
Let $\delta>0$, and define:
$$
  \epsilon_n := \sqrt{\frac{\alpha}{\delta n}} \cdot  \left(\frac{\max\{1,\tfrac{1-\alpha}{\alpha}\}}{1-\sqrt{\frac{\alpha}{\delta n}}}\right).
$$
Then, at every $n>\tfrac{\alpha}{\delta}$, we have that with probability greater than $1-\delta$:
$$
\left|\frac{\hat\alpha_n}{\alpha}-1\right|\leq \epsilon_n\quad\textrm{ and }\quad \left|\frac{1-\hat\alpha_n}{1-\alpha}-1\right|\leq \epsilon_n.
$$

If we let $\eta_n=\epsilon_n/(1-\epsilon_n)$, we can also write this as
$$
\frac{1}{1+\eta_n} \leq \frac{\hat\alpha_n}{\alpha} \leq 1+\eta_n\quad\textrm{ and }\quad\frac{1}{1+\eta_n} \leq \frac{1-\hat\alpha_n}{1-\alpha} \leq 1+\eta_n.
$$
\end{lemma}
\begin{proof}
From Chebyshev's inequality, we know that for all $\delta>0$:
$$
  \Prob\left\{|\hat\mu_n-\mu|\leq \frac{\sigma}{\sqrt{\delta n}}\right\}\geq 1-\delta.
$$ 
We now simply have to verify that $|\hat\mu_n-\mu|\leq \frac{\sigma}{\sqrt{\delta n}}$ implies that both $\left|\frac{\hat\alpha_n}{\alpha}-1\right|$ and $\left|\frac{1-\hat\alpha_n}{1-\alpha}-1\right|$ are smaller than $\epsilon_n$. Indeed, using $\hat \mu_n \geq \mu-\frac{\sigma}{\sqrt{\delta n}}$:
$$
\left|\frac{\hat\alpha_n}{\alpha}-1\right| = \left|\frac{(\hat \mu_n-1)\mu}{\hat \mu_n (\mu-1)}-1\right| = \left|(\hat \mu_n-\mu)\frac{1}{\hat\mu_n(\mu-1)}\right|\leq \left|\hat \mu_n-\mu\right| \frac{1}{(\mu-\frac{\sigma}{\sqrt{\delta n}})(\mu-1)}
$$
and
$$
\left|\frac{1-\hat\alpha_n}{1-\alpha}-1\right| = \left|\frac{\mu}{\hat \mu_n}-1\right| = \left|(\mu-\hat \mu_n)\frac{1}{\hat\mu_n}\right|\leq \left|\hat \mu_n-\mu\right| \frac{1}{(\mu-\frac{\sigma}{\sqrt{\delta n}})}.
$$

Finally, since $\left|\hat \mu_n-\mu\right|\leq \frac{\sigma}{\sqrt{\delta n}}$, both of these bounds are smaller than:
$$
 \frac{\sigma}{\sqrt{\delta n}} \frac{1}{(\mu-\frac{\sigma}{\sqrt{\delta n}})\min\{1,\mu-1\}} = \frac{\frac{\sqrt{\alpha}}{1-\alpha}}{\sqrt{\delta n}} \frac{1}{(\frac{1}{1-\alpha}-\frac{\sqrt{\alpha}}{1-\alpha}\frac{1}{\sqrt{\delta n}})\min\{1,\frac{\alpha}{1-\alpha}\}},
$$
which is equal to $\epsilon_n$. The expression with $\eta_n$ follows from $1-\epsilon_n = \tfrac{1}{1+\eta_n}$ and $1+\eta_n > 1+\epsilon_n$.
\end{proof}

It follows from Lemma \ref{lemma:parameter} that with probability greater than $1-\delta$, we have the following pointwise convergence of the distribution.
$$
(1+\eta_n)^{-x} (1-\alpha) \alpha^{x-1} \leq \hat p_\alpha(x) \leq (1+\eta_n)^x (1-\alpha) \alpha^{x-1}.
$$

Since the rate of this convergence is not uniform, we need to exercise care when specializing to particular events. We focus on the missing symbols' event. We have:
\begin{equation} \label{eq:main-ratio-bound}
\frac{\sum_{x\in E_n} (1+\eta_n)^{-x} (1-\alpha) \alpha^{x-1}}{\sum_{x\in E_n}  (1-\alpha) \alpha^{x-1}} \leq \frac{\check M_n}{M_n}=\frac{\check p_n(E_n)}{p_\alpha(E_n)} \leq
\frac{\sum_{x\in E_n} (1+\eta_n)^x (1-\alpha) \alpha^{x-1}}{\sum_{x\in E_n}  (1-\alpha) \alpha^{x-1}}.
\end{equation}

The event $E_n$ is inconvenient to sum over, because it has points spread out randomly. This is particularly true for its initial portion, where the samples ``puncture'' it. It it is more convenient to approximate this segment in order to bound Equation \eqref{eq:main-ratio-bound}. We now formalize this notion, via the following definition.

\begin{definition}[Punctured Segment] \label{def:punctured}
The punctured segment of a sample is the part between the end of the first contiguous coverage and the end of the total coverage. Its extremities are:
$$
  V_n^- := \min E_n \quad\textrm{ and }\quad V_n^+ := \max E_n^\c.
$$
\end{definition}

We have the following localization property for the punctured segment of samples from a geometric distribution.

\begin{lemma}[Localization of Punctured Segment] \label{lemma:localization} Let $X_1,\cdots,X_n$ be samples from a geometric distribution $p_\alpha(x)=(1-\alpha)\alpha^{x-1}$ on $\N$. Let $V_n^-$ and $V_n^+$ be the extremities of the punctured segment as defined in Definition \ref{def:punctured}. Then, for all $u>(\tfrac{\alpha}{1-\alpha})^2$, we have:
\begin{equation*}
 \begin{aligned}
   \Prob\{ V_n^- < \log_{1/\alpha}(n)-\log_{1/\alpha} (u) \} &< 2\mathrm{e}^{-\frac{1-\alpha}{\alpha}u} < \frac{\alpha}{(1-\alpha)u},\\
   \Prob\{ V_n^+ > \log_{1/\alpha}(n)+1+\log_{1/\alpha} (u) \} &< \frac{1}{u}.
 \end{aligned}
\end{equation*}
In particular, for $\delta<(1-\alpha)/\alpha^2$, we have that with probability greater than $1-\delta$:
$$
\log_{1/\alpha}(n)-\log_{1/\alpha}\left[\tfrac{1}{(1-\alpha)\delta}\right] \leq V_n^- < V_n^+ \leq \log_{1/\alpha}(n)+1+\log_{1/\alpha}\left[\tfrac{1}{(1-\alpha)\delta}\right].
$$
\end{lemma}
\begin{proof}

Given an integer $a\in\N$, the event that $V_n^-<a$ implies that one of the values below $a$ did not appear in the sample. By using the union bound, we thus have that:
\begin{eqnarray*}
\Prob\{V_n^-<a\}
&\leq& \sum\nolimits_{x=1}^{a-1} \left[1-(1-\alpha)\alpha^{x-1}\right]^n \\
&\leq& \sum\nolimits_{\ell=1}^{\infty} \left[1-\frac{(1-\alpha)n\alpha^{a-1-\ell}}{n}\right]^n \\
&\leq& \sum\nolimits_{\ell=1}^{\infty} \exp\left[-(1-\alpha)n\alpha^{a-1-\ell}\right]
\end{eqnarray*}

By specializing to $a(u,n)=\left\lfloor \log_{1/\alpha}(n)+1-\log_{1/\alpha} (u) \right\rfloor$:
\begin{eqnarray*}
\Prob\{ V_n^- < \log_{1/\alpha}(n)-\log_{1/\alpha} (u) \}
&\leq& \Prob\{V_n^-<a(u,n)\} \\
&\leq& \sum\nolimits_{\ell=1}^{\infty} \exp\left[-(1-\alpha)n\alpha^{\log_{1/\alpha}(n)-\log_{1/\alpha} (u)-\ell}\right] \\
&=& \sum\nolimits_{\ell=1}^{\infty} \exp\left[-(1-\alpha)\alpha^{-\ell}u\right].
\end{eqnarray*}

Lastly, if $u>(\tfrac{\alpha}{1-\alpha})^2$, one can show by induction that $(1-\alpha)\alpha^{-\ell}u>\frac{1-\alpha}{\alpha}u+\ell-1$. This turns the sum into a geometric series, giving:
\begin{equation*}
\Prob\{ V_n^- < \log_{1/\alpha}(n)-\log_{1/\alpha} (u) \}
\leq \mathrm{e}^{-\frac{1-\alpha}{\alpha}u}\sum\nolimits_{\ell=1}^{\infty} \mathrm{e}^{-\ell+1} < 2\mathrm{e}^{-\frac{1-\alpha}{\alpha}u} < \frac{\alpha}{(1-\alpha)u}.
\end{equation*}

Next, note that $V_n^+$ is nothing but the maximum of the samples. Thus, given an integer $b\in\N$, the event $V_n^+ > b$ is the complement of the event that all the samples are at $b$ or below. Since the total probability of the range $1,\cdots,b$ is $1-\alpha^b$, we thus have:
\begin{equation*}
 \Prob\{ V_n^+ > b\} = 1-(1-\alpha^b)^n.
\end{equation*}
If we now specialize to $b(u,n)=\left\lceil \log_{1/\alpha}(n)+\log_{1/\alpha} (u) \right\rceil$, we have that:
\begin{eqnarray*}
\Prob\{ V_n^+ > \log_{1/\alpha}(n)+1+\log_{1/\alpha} (u) \}
&\leq& \Prob\{ V_n^+ > b(u,n)\} \\
  &\leq& 1-\left(1-\alpha^{\log_{1/\alpha}(n)+\log_{1/\alpha} (u)}\right)^n\\
  &=& 1-\left(1-\frac{1}{u\cdot n}\right)^n < \frac{1}{u}.
\end{eqnarray*}

For the last part of the claim, we let $u=\frac{1}{(1-\alpha)\delta}$, followed by a union bound on the analyzed events. This gives us that at least one of the two events holds with probability at most $\frac{1}{u}+\frac{\alpha}{(1-\alpha)u}=\delta$, and therefore neither holds with probability at least $1-\delta$, as desired.
\end{proof}

\paragraph{(Completing the proof)} We now put together the pieces of the proof of Theorem \ref{thm:geometric-family}. To show that our estimator PAC-learns the missing mass in relative error with respect to $\mathcal{G}$, we obtain the following equivalent statement. Fix $\delta>0$ and $\eta>0$. We prove that for $n$ large enough with probability greater than $1-2\delta$ we have:
$$
\frac{1}{1+\eta} < \frac{\check M_n}{M_n} < 1+\eta.
$$

Without loss of generality, to satisfy the conditions of Lemmas \ref{lemma:parameter} and \ref{lemma:localization}, we restrict ourselves to $\delta<(1-\alpha)/\alpha^2$ (we can always choose a smaller $\delta$ than specified) and $n>\tfrac{\alpha}{\delta}$ (we can always ask for $n$ to be larger). As such, we have that with probability at least $1-2\delta$, both events of Lemmas \ref{lemma:parameter} and \ref{lemma:localization} occur. We work under the intersection of these events.

We give the details of only the right tail of the convergence; all the steps can be directly paralleled for the left tail. To see why the punctured set is a useful notion, we claim that the following quantity upper bounds the right tail of Equation \eqref{eq:main-ratio-bound}:
\begin{eqnarray}
\frac{\sum_{x> V_n^+} (1+\eta_n)^x (1-\alpha) \alpha^{x-1}}{\sum_{x> V_n^+}   (1-\alpha) \alpha^{x-1}} &=& (1+\eta_n)^{V_n^+} \frac{\sum_{y\in\N} (1+\eta_n)^y (1-\alpha) \alpha^{y-1}}{\sum_{y\in\N}  (1-\alpha) \alpha^{y-1}=1}\nonumber\\\label{eq:upper-bound}
&=& (1+\eta_n)^{V_n^+} \frac{(1-\alpha)(1+\eta_n)}{1-\alpha(1+\eta_n)}.
\end{eqnarray}
where for the first equality we have used the change of variable $y=x-V_n^+$ and simplified the common $\alpha$ factors in the numerator and denominator, and for the second equality we have used the moment generating function of the geometric distribution: $\E[\mathrm{e}^{sX}]=(1-\alpha)\mathrm{e}^s/(1-\alpha\mathrm{e}^s)$. To prove this claim, we proceed by induction, starting at step $t=1$ with the set $G^{(1)}:=\{V_n^+ +1,V_n^+ +2,\cdots\} \subset E_n$, adding at every step $t$ the largest element $z^{(t)}$ of $E_n$ not yet in $G^{(t-1)}$ to obtain $G^{(t)}$, and proving that:
$$
\frac{\sum_{x\in G^{(t)}} (1+\eta_n)^x (1-\alpha) \alpha^{x-1}}{\sum_{x\in G^{(t)}} (1-\alpha) \alpha^{x-1}} \leq \frac{\sum_{x\in G^{(t-1)}} (1+\eta_n)^x (1-\alpha) \alpha^{x-1}}{\sum_{x\in G^{(t-1)}} (1-\alpha) \alpha^{x-1}}.
$$

We use the following basic property that for positive real numbers $a_1,b_1,a_2,b_2$, the following three equalities are equivalent:
\begin{equation*}
 \begin{array}{rrcl}
    (i)&\quad a_1/b_1 		&\leq& a_2/b_2,\\
   (ii)&\quad a_1/b_1 		&\leq& (a_1+a_2)/(b_1+b_2),\\
  (iii)&\quad (a_1+a_2)/(b_1+b_2)&\leq& a_2/b_2.
 \end{array}
\end{equation*}

For the base case, let $a_2=\sum_{x\in G^{(1)}} (1+\eta_n)^x (1-\alpha) \alpha^{x-1}$ and $b_2=\sum_{x\in G^{(1)}} (1-\alpha) \alpha^{x-1}$. We then choose the largest $z^{(1)}\in E_n \setminus G^{(1)}$ and we let $a_1=(1+\eta_n)^{z^{(1)}} (1-\alpha) \alpha^{z^{(1)}-1}$ and $b_1=(1-\alpha) \alpha^{z^{(1)}-1}$. From \eqref{eq:upper-bound}, noting that the fraction is always greater than $1$, it follows that $a_2/b_2 > (1+\eta_n)^{V_n^+}>(1+\eta_n)^{z^{(1)}} = a_1/b_1$. We can thus add $z^{(1)}$ to the sum, and obtain $(a_1+a_2)/(b_1+b_2)\leq a_2/b_2$, establishing the base case. Note that this also shows that $(a_1+a_2)/(b_1+b_2)\geq a_1/b_1 = (1+\eta_n)^{z^{(1)}}$. We pass this property down by induction, and we can assume this holds true at every step.

To continue the induction at step $t$, let $a_2=\sum_{x\in G^{(t-1)}} (1+\eta_n)^x (1-\alpha) \alpha^{x-1}$ and $b_2=\sum_{x\in G^{(t-1)}} (1-\alpha) \alpha^{x-1}$. As noted, we assume that $a_2/b_2 \geq (1+\eta_n)^{z^{(t-1)}}$ from the previous induction step. We then choose the largest $z^{(t)}\in E_n \setminus G^{(t-1)}$ and we let $a_1=(1+\eta_n)^{z^{(t)}} (1-\alpha) \alpha^{z^{(t)}-1}$ and $b_1=(1-\alpha) \alpha^{z^{(t)}-1}$. Since $z^{(t-1)}<z^{(t)}$, it follows that $a_2/b_2 \geq (1+\eta_n)^{z^{(t-1)}} > (1+\eta_n)^{z^{(t)}} = a_1/b_1$. We can thus add $z^{(t)}$ to the sum, and obtain $(a_1+a_2)/(b_1+b_2)\leq a_2/b_2$, as desired. Note  that this also shows that $(a_1+a_2)/(b_1+b_2)\geq a_1/b_1 = (1+\eta_n)^{z^{(t)}}$, and the induction is complete.

By combining this result with the equivalent argument on the left side, we have effectively shown that we can replace Equation \eqref{eq:main-ratio-bound} by
\begin{equation*}
\frac{\sum_{x\geq V_n^-} (1+\eta_n)^{-x} (1-\alpha) \alpha^{x-1}}{\sum_{x\geq V_n^-}  (1-\alpha) \alpha^{x-1}} \leq \frac{\check M_n}{M_n}=\frac{\check p_n(E_n)}{p_\alpha(E_n)} \leq
\frac{\sum_{x> V_n^+} (1+\eta_n)^x (1-\alpha) \alpha^{x-1}}{\sum_{x> V_n^+}  (1-\alpha) \alpha^{x-1}}
\end{equation*}
or equivalently by
\begin{equation} \label{eq:new-ratio-bound}
(1+\eta_n)^{-V_n^-+1} \frac{(1-\alpha)(1+\eta_n)^{-1}}{1-\alpha(1+\eta_n)^{-1}} \leq \frac{\check M_n}{M_n} \leq
(1+\eta_n)^{V_n^+} \frac{(1-\alpha)(1+\eta_n)}{1-\alpha(1+\eta_n)}.
\end{equation}


In Lemma \ref{lemma:parameter} we have set:
$$
\eta_n=\epsilon_n/(1-\epsilon_n),
$$
with
$$
  \epsilon_n := \sqrt{\frac{\alpha}{\delta n}} \cdot  \left(\frac{\max\{1,\tfrac{1-\alpha}{\alpha}\}}{1-\sqrt{\frac{\alpha}{\delta n}}}\right).
$$

On the other hand, by Lemma \ref{lemma:localization}, we have that:
$$
V_n^+ \leq \log_{1/\alpha}(n)+1+\log_{1/\alpha}\left[\tfrac{1}{(1-\alpha)\delta}\right]
$$
and
$$
V_n^- \geq \log_{1/\alpha}(n)-\log_{1/\alpha}\left[\tfrac{1}{(1-\alpha)\delta}\right].
$$

It follows that both bounds of Equation \eqref{eq:new-ratio-bound} converge to $1$, at the rate of roughly $\log(n)/\sqrt{n}$, instead of the parametric rate $1/\sqrt{n}$. Regardless, for any desired $\eta>0$, we get that there exists a large enough $n$ beyond which, with probability greater than $1-2\delta$, we satisfy:
$$
\frac{1}{1+\eta} \leq \frac{\check M_n}{M_n} \leq 1+\eta.
$$
This establishes that $\check M_n$ PAC-learns $M_n$, as desired. \hfill{$\blacksquare$}

\end{document}